\documentclass[nohyperref]{article}

\usepackage{microtype}
\usepackage{graphicx}
\usepackage{subfigure}
\usepackage{booktabs} 
\usepackage{balance}

\usepackage{hyperref}

\usepackage[accepted]{icml2023}

\usepackage{amsmath}
\usepackage{amssymb}
\usepackage{mathtools}
\usepackage{amsthm}
\usepackage{bbm}

\usepackage[capitalize,noabbrev]{cleveref}

\newcommand{\RR}{\mathbb{R}}

\newcommand{\DD}{\mathbb{D}}

\newcommand{\X}{\mathcal{X}}

\renewcommand{\H}{\mathcal{H}}
\newcommand{\A}{\mathcal{A}}
\newcommand{\R}{\mathcal{R}}
\newcommand{\V}{\mathcal{V}}

\newcommand{\squishlist}{
 \begin{list}{$\bullet$}
  { \setlength{\itemsep}{0pt}
     \setlength{\parsep}{3pt}
     \setlength{\topsep}{3pt}
     \setlength{\partopsep}{0pt}
     \setlength{\leftmargin}{1.5em}
     \setlength{\labelwidth}{1em}
     \setlength{\labelsep}{0.5em} } }

\newcommand{\squishend}{
  \end{list}  }

\theoremstyle{plain}
\newtheorem{theorem}{Theorem}[section]

\newtheorem{corollary}[theorem]{Corollary}
\theoremstyle{definition}
\newtheorem{definition}[theorem]{Definition}

\theoremstyle{remark}
\newtheorem{remark}[theorem]{Remark}

\usepackage{thm-restate}

\usepackage[textsize=tiny]{todonotes}

\icmltitlerunning{Fundamental Tradeoffs in Learning with Prior Information}

\begin{document}

\twocolumn[
\icmltitle{Fundamental Tradeoffs in Learning with Prior Information}




\begin{icmlauthorlist}
\icmlauthor{Anirudha Majumdar}{yyy}
\end{icmlauthorlist}

\icmlaffiliation{yyy}{Department of Mechanical and Aerospace Engineering, Princeton University, Princeton, NJ, USA}

\icmlcorrespondingauthor{Anirudha Majumdar}{ani.majumdar@princeton.edu}

\icmlkeywords{Minimax risk, Bayes risk, prioritized risk, fundamental tradeoffs}

\vskip 0.3in
]



\printAffiliationsAndNotice{} 

\begin{abstract}
We seek to understand fundamental tradeoffs between the accuracy of prior information that a learner has on a given problem and its learning performance. We introduce the notion of \emph{prioritized risk}, which differs from traditional notions of minimax and Bayes risk by allowing us to study such fundamental tradeoffs in settings where reality does not necessarily conform to the learner's prior. We present a general reduction-based approach for extending classical minimax lower-bound techniques in order to lower bound the prioritized risk for statistical estimation problems. We also introduce a novel generalization of Fano's inequality (which may be of independent interest) for lower bounding the prioritized risk in more general settings involving unbounded losses. We illustrate the ability of our framework to provide insights into tradeoffs between prior information and learning performance for problems in estimation, regression, and reinforcement learning. 
\end{abstract}

\section{Introduction}
\label{sec:intro}

We are motivated by the problem of understanding fundamental limits and tradeoffs in learning with prior information: how much prior knowledge does one \emph{require} in order to learn quickly on a given task? In other words, for a given problem, is there a fundamental limit on the performance of a learner in the absence of sufficient prior information? 
Fundamental limits in learning are typically formalized via the classical notions of \emph{minimax risk} and \emph{Bayes risk} \cite{Berger13}. Consider the standard setting in statistical learning theory where a learner receives a dataset $x_1^n \coloneqq \{x_i\}_{i=1}^n$ containing $n$ i.i.d. observations of a random variable $X$. Suppose that the distribution $P_\theta$ of $X$ is defined by a parameter $\theta \in \Theta$ which is \emph{a priori} unknown to the learner. The learner's task is to use the dataset to minimize \emph{risk} (i.e., expected loss) for this unknown distribution. The minimax risk then corresponds to the lowest \emph{worst-case} risk (over the family of distributions defined by $\theta \in \Theta$) achievable by any learner (see Sec.~\ref{sec:background} for a formal definition). In other words, the minimax risk corresponds to the minimax-optimal value of the following game: the learner first chooses a particular learning algorithm, and then an adversary chooses a distribution parameter $\theta \in \Theta$; the learner's cost (negative payoff) is equal to its risk. 

The minimax risk does not allow one to take into account prior information that a learner may have (beyond the weak prior knowledge that $\theta$ belongs to the set $\Theta$). As an alternative, one may assume that the learner is equipped with prior knowledge of the distribution $\pi$ that the parameter $\theta$ is drawn from. The \emph{Bayes risk} then corresponds to the smallest possible \emph{average} risk assuming that $\theta$ is drawn from the prior distribution $\pi$ (see Sec.~\ref{sec:background} for a formal definition). 
However, the Bayes risk formulation assumes that the learner is equipped with \emph{nature's prior}, i.e., the true distribution from which $\theta$ is drawn. The Bayes risk formulation thus does not capture settings where reality does not conform to the learner's prior. 

{\bf Statement of contributions.} We seek to address the challenges with using the minimax and Bayes risk for understanding fundamental limits of learning with prior information. To this end, we make the following contributions:
\squishlist
\item We propose a quantity --- which we refer to as \emph{prioritized risk} --- for analyzing settings where reality does not fully conform to the learner's prior. The key idea is to consider a version of the minimax risk where the risk associated with a given distribution $P_\theta$ is \emph{weighted by the prior} that the learner has on $\theta$. We show that this quantity allows us to understand fundamental tradeoffs between the accuracy of prior information and learning performance (risk) on a given problem: a lower bound on the prioritized risk allows one to establish a fundamental limit on learning performance in the absence of sufficient prior information (Sec.~\ref{sec:prioritized risk}).
\item We provide a general reduction-based strategy for extending classical techniques for lower bounding minimax and Bayes risk --- including the methods of LeCam, Assouad, and Fano \cite{Tsybakov08, Yang99, Yu97} --- to obtain lower bounds on the prioritized risk for estimation problems (Sec.~\ref{sec:lower bounds estimation}).
\item We derive a novel \emph{generalized Fano inequality} for obtaining lower bounds on the prioritized risk for general learning problems beyond estimation (Sec.~\ref{sec:generalized fano}). This inequality handles problems with unbounded losses (in contrast to prior generalized Fano inequalities \cite{Chen16, Gerchinovitz20, Majumdar22}), and may thus be of independent interest. 
\item We illustrate the ability of the prioritized risk framework to provide insights into tradeoffs between prior information and learning performance for various problems including prior-informed versions of the following (Sec.~\ref{sec:examples}): (i) Bernoulli mean estimation, (ii) logistic regression, and (iii) reinforcement learning (RL) with environments drawn from Zipfian distributions. 
\squishend

\section{Background: minimax and Bayes risk}
\label{sec:background}

We provide a brief introduction to the minimax and Bayes risks, and refer the reader to \cite{Berger13, Tsybakov08} for a more thorough exposition. 
Consider a random variable $X$ that takes values in a sample space $\X$. Suppose that the distribution $P_\theta$ of $X$ is defined by a parameter $\theta \in \Theta$ which is unknown to the learner. A learner $\sigma: \X^n \rightarrow \A$ receives a dataset $x_1^n \coloneqq \{x_i\}_{i=1}^n$ of $n$ i.i.d. realizations of $X$, and must output an \emph{action} $a \in \A$ that is evaluated according to a loss function $L: \Theta \times \A \rightarrow [0,\infty)$. This setting captures many problems of interest; in estimation problems, one can take $\A = \Theta$ and evaluate the learner using a loss $L(\theta, \sigma(x_1^n)) \coloneqq (\theta - \sigma(x_1^n))^2$. More broadly, $\A$ can correspond to a hypothesis space $\H$; the learner may then be evaluated by its future expected performance $L(\theta, \sigma(x_1^n)) \coloneqq \mathbb{E}_{X \sim P_\theta} [l(X, \sigma(x_1^n))]$ on data from $P_\theta$. 

{\bf Minimax risk.} The \emph{risk} of a learner is defined as
\begin{equation}
    \R(\sigma, \theta) \coloneqq \underset{X_1^n \sim P_\theta^n}{\mathbb{E}} \big{[} L(\theta, \sigma(X_1^n)) \big{]},
\end{equation}
where the expectation is taken with respect to the dataset used by the learner. We can then define the \emph{minimax risk}: 
\begin{align}
\label{eq:minimax risk}
\R_\text{minimax}(L; \Theta) &\coloneqq \inf_{\sigma} \ \sup_{\theta \in \Theta} \ \R(\sigma, \theta). 
\end{align}
This can be interpreted as a game where the learner fixes $\sigma$, and an adversary with knowledge of $\sigma$ then chooses $\theta \in \Theta$ in order to maximize the risk. The minimax risk thus corresponds to the lowest \emph{worst-case} risk (over the family of distributions defined by $\theta \in \Theta$) of any learner. 

{\bf Bayes risk.} Since the minimax risk corresponds to the worst-case risk, it does not capture prior knowledge that a learner may have beyond the fact that $\theta \in \Theta$. One can instead consider the Bayes risk, which is the smallest possible average risk assuming that $\theta$ is drawn from a distribution $\pi$:
\begin{align}
\label{eq:bayes risk}
\R_\text{Bayes}(\pi, L; \Theta) &\coloneqq \inf_{\sigma} \ \int_{\Theta} \R(\sigma, \theta)  \ \pi(d\theta) \ . 
\end{align}
A learner that implements Bayesian inference with the prior $\pi$ and observation $x_1^n$ achieves the optimal Bayes risk \citep[Ch. 4]{Berger13}. However, the Bayes risk formulation makes a strong assumption on how $\theta$ is chosen and the knowledge that the learner has, i.e., that the parameter $\theta$ is drawn from a ``true" distribution $\pi$ (``nature's prior") and that this distribution is \emph{known} to the learner. The Bayes risk is thus not directly useful in analyzing settings where reality does not conform to the learner's prior. 
\section{Prioritized risk}
\label{sec:prioritized risk}

Motivated by the challenges associated with the notions of minimax and Bayes risk for analyzing learning algorithms with prior information, we propose a different quantity that aims to capture the relationship between prior knowledge and learning performance on a given task. We refer to this quantity as \emph{prioritized risk} and discuss its interpretation below. Similar to the minimax risk, the prioritized risk operates in a setting where nature chooses a \emph{particular} value $\theta \in \Theta$ (instead of randomly choosing $\theta$ from a distribution). However, similar to the Bayes risk, we allow learners to be equipped with prior information (which may not fully capture the true value of $\theta$).  

\subsection{Definition}

Let $\pi: \Theta \rightarrow \RR_{>0}$ denote a function that captures prior information that a learner has about the learning problem (i.e., about the value of $\theta \in \Theta$ that the learner will encounter). If $\pi$ is normalized such that $\int_\Theta \pi(\theta) d\theta = 1$, it may be interpreted as a density corresponding to a Bayesian prior. However, here we eschew this Bayesian interpretation and simply think of $\pi$ as a way for the learner to encode all available prior knowledge or inductive bias it has on the problem (specified before the learner observes any data). In this sense, $\pi$ is similar to the ``luckiness function" \cite{Shawe-Taylor98} or the prior in PAC-Bayes approaches \cite{McAllester99}. In cases where $\pi$ does not integrate to 1, it may be interpreted as an energy-based model \cite{Lecun06} which encodes prior information. Our framework allows us to handle both normalized and unnormalized priors.

\vspace{5pt}

\begin{definition}[Prioritized risk] For a given family of distributions $\{P_\theta\}_{\theta \in \Theta}$, loss function $L$ , and prior function $\pi: \Theta \rightarrow \RR_{>0}$, the \emph{prioritized risk} is defined as:
\begin{align}
\label{eq:prioritized risk}
    \R_\text{prior}(\pi, L; \Theta) &\coloneqq \inf_{\sigma} \ \sup_{\theta \in \Theta} \ \pi(\theta) \R(\sigma, \theta) \\
&= \inf_{\sigma} \ \sup_{\theta \in \Theta} \ \pi(\theta) \underset{X_1^n \sim P_\theta^n}{\mathbb{E}} \big{[} L(\theta, \sigma(X_1^n)) \big{]}. \nonumber
\end{align}
We will refer to the quantity:
\begin{align}
\label{eq:learner-specific prioritized risk}
\R^\sigma_\text{prior}(\pi, L; \Theta) \coloneqq  \ \sup_{\theta \in \Theta} \ \pi(\theta) \R(\sigma, \theta)
\end{align}
as the \emph{learner-specific prioritized risk}. 
\end{definition}

As we discuss below, the prioritized risk can provide insights along two dimensions: (i) analyzing different learning \emph{algorithms}, and (ii) analyzing different learning \emph{problems}. 

\subsection{Implications for learning algorithms}
\label{sec:implications for learning algorithms}

In order to interpret the prioritized risk, consider a learner $\sigma$ that achieves a small learner-specific prioritized risk:
\begin{equation}
\label{eq:learner-specific prioritized risk upper bound}
    \pi(\theta) \R(\sigma, \theta) \leq \epsilon, \ \forall \theta \in \Theta.
\end{equation}
We will say that a parameter $\theta \in \Theta$ chosen by nature \emph{conforms} to the learner's prior $\pi$ if $\pi(\theta)$ is high; conversely, we will say that $\theta$ does not conform to the learner's prior if $\pi(\theta)$ is low.
Then, we have the following implication for a learner that satisfies \eqref{eq:learner-specific prioritized risk upper bound}: the more closely nature conforms to the learner's prior (i.e., the higher $\pi(\theta)$ is for the chosen $\theta$), the lower the risk is guaranteed to be: $\R(\sigma, \theta) \leq \epsilon / \pi(\theta)$. 

The prioritized risk also allows us to compare different learning algorithms. For a given learning problem and prior $\pi$, consider two learners $\sigma$ and $\sigma'$ that have learner-specific prioritized risks $\epsilon$ and $\epsilon'$ respectively, with $\epsilon' < \epsilon$. Then,
\begin{align}
    &\R(\sigma, \theta) \leq \epsilon / \pi(\theta), \ \forall \theta \in \Theta, \\
    &\R(\sigma', \theta) \leq \epsilon' / \pi(\theta) < \epsilon / \pi(\theta), \ \forall \theta \in \Theta. 
\end{align}
In such a case, one should prefer the learner $\sigma'$ since it affords a better tradeoff between learning performance (i.e., risk) and the accuracy of prior information (i.e., how much $\theta$ conforms to the prior). 

\subsection{Implications for learning problems}
\label{sec:implications for learning problems}

In this work, we will focus on \emph{lower bounds} for the prioritized risk. Consider a problem where one has established:
\begin{equation}
\label{eq:prioritized risk lower bound}
\R_\text{prior}(\pi, L; \Theta) \geq \beta.
\end{equation}
This implies\footnote{Here we assume for simplicity that the supremum and infimum are achieved (i.e., sup = max, inf = min).} that no matter what learning algorithm one chooses (i.e., for any choice of $\sigma$),
\begin{equation}
\label{eq:prioritized risk lower bound 2}
\exists \theta \in \Theta \ \text{such that} \ \underset{\begin{subarray}{c} \vspace{1.5pt} \text{Prior} \\ \text{(``Nature")}\end{subarray}}{\underbrace{\pi(\theta)}} \  \underset{\begin{subarray}{c}\text{Learning} \\ \text{(``Nurture")}\end{subarray}}{\underbrace{\R(\sigma, \theta)}} \geq \beta. 
\end{equation}
This relationship between prior information (``nature") and learning (``nurture") takes the form of an \emph{uncertainty principle}\footnote{Recall the form of uncertainty principles in quantum mechanics, e.g., $\Delta x \Delta p \geq h/4\pi$, which states that there is a fundamental tradeoff between knowing a particle's position and its momentum.} and captures a fundamental tradeoff: it is \emph{impossible} for both the risk and the prior to be low for all $\theta \in \Theta$. In other words, for any learning algorithm $\sigma$, there exists $\theta \in \Theta$ such that if the learner achieves low risk, it \emph{must} be the case that reality conforms to the learner's prior (i.e., $\pi(\theta)$ is large). 

\emph{Example \ref{sec:implications for learning problems}.} As an example of the kind of implications one may derive from a lower bound \eqref{eq:prioritized risk lower bound} on the prioritized risk, consider a learning problem with an associated prior $\pi$ (with $\pi(\theta) \in [0,1], \forall \theta \in \Theta$). Consider a learner $\sigma$ that achieves low risk for values of $\theta$ that have a high prior:
\begin{equation}
    \R(\sigma, \theta) \leq \epsilon < \beta , \ \ \forall \theta \ \text{s.t.} \ \frac{1}{2} \leq \pi(\theta) \leq 1. 
\end{equation}
Then, from \eqref{eq:prioritized risk lower bound 2}, we see that there must exist a $\theta$ with low prior where the learner performs poorly:
\begin{equation}
    \exists \theta \ \text{s.t.} \ \pi(\theta) < \frac{1}{2}, \ \text{where} \ \R(\sigma, \theta) \geq \frac{\beta}{\pi(\theta)} > 2\beta. 
\end{equation}
A lower bound on the prioritized risk thus establishes a fundamental tradeoff for a given learning problem: any learner that performs extremely well for values of $\theta$ with high $\pi(\theta)$ \emph{must} give up performance for a low value of $\pi(\theta)$; the learner may thus perform poorly if reality does not conform to its prior (i.e., if $\pi(\theta)$ is low). 

\vspace{2pt}
\begin{remark} [Relationship to minimax and Bayes risk]
\label{remark:relationship to minimax and bayes risk}
We make the following straightforward observations relating the prioritized risk to the minimax and Bayes risks:
\squishlist 
\item The prioritized risk reduces to the minimax risk if $\pi(\theta) \equiv 1$. If $\pi(\theta) \leq 1 \ \forall \theta$, then $\R_\text{prior} \leq \R_\text{minimax}$. 

\item In settings where $\Theta$ is a countable set, the prioritized risk lower bounds the Bayes risk (both computed using a given prior $\pi$ that is normalized to be a valid probability distribution). This is because for any learner $\sigma$: $\sup_{\theta \in \Theta} \ \pi(\theta) \R(\sigma, \theta) \leq \sum_{\theta \in \Theta} \pi(\theta) \R(\sigma, \theta)$. 
Since the Bayes risk lower bounds the minimax risk, we have for countable $\Theta$ that: $ \R_\text{prior} \leq \R_\text{Bayes}\leq \R_\text{minimax}$. 
\item In the more general setting of uncountable $\Theta$, the prioritized risk may be larger than the Bayes risk (e.g., take $\pi$ to be the density function of a univariate Gaussian such that $\pi(\theta) > 1$ for some $\theta$, and let $\R(\sigma, \theta) \equiv 1$). 
\squishend 
\end{remark}

\section{Lower bounds on prioritized risk: estimation problems}
\label{sec:lower bounds estimation}

We now describe techniques for obtaining lower bounds on the prioritized risk. In this section, we focus on estimation problems; here, a learner $\sigma: \X^n \rightarrow \Theta$ receives a dataset $x_1^n \coloneqq 
 \{x_i\}_{i=1}^n$ of $n$ i.i.d. realizations from a distribution $P_\theta$ and outputs an estimate $\hat{\theta} \coloneqq \sigma(x_1^n)$ of the underlying parameter $\theta$. The learner is evaluated using a loss:
 \begin{equation}
 \label{eq:estimation loss}
     L(\theta, \sigma(x_1^n)) \coloneqq \rho \big{(}\theta, \sigma(x_1^n) \big{)} ,
 \end{equation}
 where $\rho: \Theta \times \Theta \rightarrow \RR_+$ is a (pseudo)metric, e.g.,  $\|\theta - \hat{\theta}\|_2$.  

\subsection{Reduction from estimation to testing}

In this section, we describe a general strategy for extending classical techniques for lower bounding the minimax risk (e.g., the methods of LeCam, Assouad, and Fano) in order to lower bound the prioritized risk for estimation problems. The standard starting point for proving lower bounds on the minimax risk is to reduce the problem of estimation to one of hypothesis testing (see, e.g., \cite{Tsybakov08} \citep[Ch. 7]{Duchi16}). One can then use information-theoretic techniques to lower bound the Bayes risk for the testing problem, which yields a lower bound for the minimax risk. Here, we extend this classical reduction from estimation to testing in order to lower bound the prioritized risk. 

Consider a family of distributions $\{\theta_v\}_{v \in \V}$, where $\V$ is a finite index set. We will refer to this set as a $(\delta, \pi)$-packing if balls (as defined by the pseudometric $\rho$) of radius $\frac{\delta}{\pi(\theta_v)}$ centered around each $\theta_v$ are non-overlapping. Now consider the following hypothesis testing problem. First, we define a random variable $V$ corresponding to a uniform distribution over $\V$. Conditioned on a choice $V=v$, a dataset $x_1^n \coloneqq \{x_i \}_{i=1}^n$ is drawn from $P_{\theta_v}^n$. Given this dataset, the hypothesis testing problem is to determine the underlying index $v \in \V$. A mapping $\Psi: \X^n \rightarrow \V$ is referred to as a \emph{test function}. Its associated probability of error is:
\begin{equation}
\mathbb{P} (\Psi (X_1^n) \neq V) \coloneqq \frac{1}{|\V|} \sum_{v \in \V} \mathbb{P} \Big{(}\Psi (X_1^n) \neq v | V = v \Big{)}. \nonumber
\end{equation}

We now establish the reduction from estimation (where a learner is evaluated according to the learner-specific prioritized risk) to hypothesis testing via the following argument:

1. Suppose there exists a learner $\sigma: \X^n \rightarrow \Theta$ with a low learner-specific prioritized risk $\R^\sigma_\text{prior}(\pi, L; \Theta)$ (Eq. \eqref{eq:learner-specific prioritized risk})
for the original estimation problem. 

2. Then, the following \emph{prior-weighted} test function (which uses $\sigma$ as a subroutine) achieves a small error probability for the hypothesis testing problem:
\begin{equation}
\label{eq:prior-weighted test function}
\Psi(x_1^n) \coloneqq \underset{v \in \V}{\text{argmin}} \ \pi(\theta_v) \rho(\theta_v, \sigma(x_1^n)).  
\end{equation}
This reduction (stated formally below) allows us to turn a learner that achieves low prioritized risk (for the estimation problem) to a test function that achieves a small probability of error (for the hypothesis testing problem). The contrapositive then allows us to translate lower bounds on the testing problem (which can be obtained using standard information-theoretic techniques) to lower bounds on the estimation problem. Specifically, suppose we have a lower bound on the achievable probability of error for the testing problem. This lower bounds the probability of error for the prior-weighted test function \eqref{eq:prior-weighted test function} (since this is just a particular test function). The reduction above then provides a lower bound on the prioritized risk for the estimation problem. 

\begin{restatable}[Reduction from estimation to testing]{proposition}{estimationreduction}
\label{prop:reduction}
Let $\sigma: \X^n \rightarrow \Theta$ be an estimator for a learning problem defined by a loss function of the form \eqref{eq:estimation loss}. Let $\{ \theta_v \}_{v \in \V}$ form a $(\delta,\pi)$-packing and define the prior-weighted test function $\Psi$ as in \eqref{eq:prior-weighted test function}. We then have the following bound:
\begin{equation}
\R^\sigma_\text{prior}(\pi, L; \Theta) \geq \delta \cdot \mathbb{P} (\Psi (X_1^n) \neq V).
\end{equation}
Hence, taking an infimum over estimators, we have:
\begin{equation}
\label{eq:reduction inequality}
\R_\text{prior}(\pi, L; \Theta) \geq \delta \ \inf_{\Psi} \ \mathbb{P} (\Psi (X_1^n) \neq V).
\end{equation}
\end{restatable}
\begin{proof}
    The proof is presented in Appendix~\ref{app:proofs}. The primary distinctions of this reduction as compared to the standard reduction from estimation to testing are: (i) the use of the \emph{non-uniform} $(\delta,\pi)$-packing, and (ii) the use of the \emph{prior-weighted} test function \eqref{eq:prior-weighted test function}. If $\pi(\theta) \equiv 1$, the reduction presented here reduces to the standard reduction.  
\end{proof}
As we demonstrate below, this reduction allows us to extend classical techniques for lower bounding the minimax risk \cite{Tsybakov08, Yang99, Yu97} in order to lower bound the prioritized risk. Proofs of the results below are deferred to Appendix~\ref{app:proofs}. 

\subsection{LeCam's method for prioritized risk}

We first consider an extension of LeCam's method in order to lower bound the prioritized risk. This technique operates by constructing a packing $\{\theta_v\}_{v \in \V}$, where $\V = \{0,1\}$, and utilizing lower bounds for binary hypothesis testing. 

\begin{restatable}[LeCam's method for prioritized risk]{theorem}{lecam}
\label{thm:lecam}
Consider an estimation problem over a family of distributions $\{P_\theta\}_{\theta \in \Theta}$ with a loss function of the form \eqref{eq:estimation loss}. A learner has prior $\pi$ and will receive $n$ i.i.d. observations from a chosen distribution. Let $\{\theta_0, \theta_1 \}$ form a $(\delta, \pi)$-packing. We then have the following lower bound on the prioritized risk:
\begin{equation}
    \R_\text{prior}(\pi, L; \Theta) \geq \frac{\delta}{2} \Big{[} 1 - \|P_{\theta_0}^n - P_{\theta_1}^n \|_\text{TV} \Big{]}.
\end{equation}
\end{restatable}

In order to understand the dependence of $\R_\text{prior}$ on the number of samples $n$, we can find a $(\delta(n), \pi)$-packing as a function of $n$ such that the total variation distance is bounded:
\begin{equation}
    \|P_{\theta_0}^n - P_{\theta_1}^n\|_\text{TV} \leq 1 - \frac{2}{\delta(n) \lambda(n)}. 
\end{equation}
Theorem~\ref{thm:lecam} then establishes that $\R_\text{prior} \geq 1 / \lambda(n)$. 

\subsection{Fano's method for prioritized risk}
\label{sec:fano}

Next, we extend Fano's method in order to lower bound the prioritized risk; this operates by lower bounding the testing error \eqref{eq:reduction inequality} via Fano's inequality \cite{Cover12}. 

\begin{restatable}[Fano's method for prioritized risk]{theorem}{fano}
Consider an estimation problem over a family of distributions $\{P_\theta\}_{\theta \in \Theta}$ with a loss function of the form \eqref{eq:estimation loss}. A learner has prior $\pi$ and will receive $n$ i.i.d. observations from a chosen distribution. Let $\{\theta_v\}_{v \in \V}$ form a $(\delta, \pi)$-packing. Define a random variable $V$ corresponding to a uniform distribution over $\V$. We can then lower bound the prioritized risk:
\begin{equation}
    \R_\text{prior}(\pi, L; \Theta) \geq \delta \Bigg{[} 1 - \frac{I(V;X_1^n) + \log(2)}{\log |\V|} \Bigg{]},
\end{equation}
where $I(V; X_1^n)$ denotes the mutual information. 
\end{restatable}

\subsection{Assouad's method for prioritized risk}

We now extend Assouad's method in order to lower bound $\R_\text{prior}$; instead of reducing the problem of estimation to a single hypothesis testing problem as in Prop.~\ref{prop:reduction}, Assouad's method proceeds via a reduction to multiple binary hypothesis testing problems. We first extend the notion of \emph{Hamming separation} \citep[Ch. 7]{Duchi16} to incorporate the prior $\pi$. 

\begin{definition}[$(2\delta,\pi)$-Hamming separation]
\label{defn:hamming separation}
Consider a family of distributions given by $\{\theta_v\}_{v \in \V}$ indexed by the hypercube $\V = \{-1,1\}^d$ (for some $d \in \mathbb{N}$). This family induces a \emph{$(2\delta,\pi)$-Hamming separation} if there exists $\hat{v}: \Theta \rightarrow \{-1,1\}^d$ such that $\forall v \in \V$, we have:
\begin{equation}
    \rho(\theta_v, \theta) \geq \frac{2\delta}{\pi(\theta_v)} \sum_{i=1}^d \mathbbm{1} \Big{\{} [\hat{v}(\theta)]_j \neq v_j \Big{\}}, \ \forall \theta \in \Theta.
\end{equation}
Let $\mathbb{P}_{\pm j}$ denote the joint distribution over the (uniformly chosen) random index $V$ and data $X_1^n$ conditioned on the $j$-th coordinate $V_j = \pm 1$.
\end{definition}

\begin{restatable}[Assouad's method for prioritized risk]{theorem}{assouad}
\label{thm:assouad}
Consider an estimation problem over a family of distributions $\{P_\theta\}_{\theta \in \Theta}$ with a loss function of the form \eqref{eq:estimation loss}. A learner has prior $\pi$ and will receive $n$ i.i.d. observations from a chosen distribution. Let $\{\theta_v\}_{v \in \V= \{-1,1\}^d}$ form a $(2\delta, \pi)$-Hamming separation with $\hat{v}$. We then have:
\begin{align}
\R_\text{prior}(\pi, L; \Theta) \geq \delta \sum_{i=1}^d \inf_\Psi \ \Big{[} &\mathbb{P}_{+j} \big{(} \Psi(X_1^n) \neq +1 \big{)} + \dots \nonumber \\
&\mathbb{P}_{-j} \big{(} \Psi(X_1^n) \neq -1 \big{)} \Big{]}, \nonumber
\end{align}
where the infimum is over tests $\Psi: \X^n \rightarrow \{+1,-1\}$. 
\end{restatable}

Combining this with the variational representation of the total variation distance, we see:
\begin{equation}
\label{eq:assouad tv}
    \R_\text{prior}(\pi, L; \Theta) \geq \delta \sum_{j=1}^d \Big{[} 1 - \big{\|} P_{{+j}}^n - P_{{-j}}^n \big{\|}_\text{TV} \Big{]},
\end{equation}
where $P_{{+j}} \coloneqq 2^{1-d} \sum_{v|v_j=1} P_{\theta_v}$ (and similarly for $P_{-j}$).
Thus, similar to LeCam's method, we can obtain lower bounds on the prioritized risk by finding an appropriate packing (forming a $(2\delta,\pi)$-Hamming separation) and lower bounding the total variation distances in \eqref{eq:assouad tv}.

\section{Generalized Fano inequality for lower bounds on prioritized risk}
\label{sec:generalized fano}

We now describe techniques for lower bounding the prioritized risk for learning problems beyond estimation. We consider the general setting described in Sec.~\ref{sec:background}, where a learner $\sigma: \X^n \rightarrow \A$ receives a dataset $x_1^n \coloneqq \{x_i\}_{i=1}^n$ of $n$ i.i.d. realizations of $X$, and must output an \emph{action} $a \in \A$ (e.g., a hypothesis) that is evaluated according to a loss function $L: \Theta \times \A \rightarrow [0,\infty)$.

The key technical tool we use to obtain lower bounds on prioritized risk in this setting is a \emph{generalized} version of Fano's inequality. 
In its original form, Fano's inequality \cite{Cover12} provides lower bounds on the achievable error of estimating a signal given a potentially corrupted observation of the signal. Recent generalized Fano inequalities \cite{Chen16, Gerchinovitz20, Majumdar22} allow one to establish lower bounds on Bayes and minimax risks for various learning problems beyond estimation. Here, we present a novel generalization of Fano's inequality, which allows us to handle learning problems with \emph{unbounded} loss functions (in contrast to \citet{Chen16, Gerchinovitz20, Majumdar22}, which assume that the loss is bounded within $[0,1]$). This is particularly important in our setting for lower bounding the prioritized risk \eqref{eq:prioritized risk}, where the product $\pi(\theta) \R(\sigma, \theta)$ may not be bounded (even if the loss function $L$ is bounded). We first present a general version of our result --- which may be of independent interest --- for lower bounding the Bayes risk \eqref{eq:bayes risk}, and then use this to lower bound the prioritized risk. 

\begin{restatable}[Generalized Fano inequality for unbounded losses]{theorem}{generalizedfano}
\label{thm:generalized fano}
For any prior distribution $p$ on $\theta$, the following lower bound on the Bayes risk holds for all $\lambda > 0$:
\begin{align}
&\R_\text{Bayes}(p, L; \Theta) \coloneqq \inf_\sigma \ \int_{\Theta} \underset{X_1^n \sim P_\theta^n}{\mathbb{E}}  \big{[} L(\theta, \sigma(X_1^n)) \big{]}  \ p(d\theta) \nonumber \\
&\geq \frac{1}{\lambda} \Big{[} \rho_{\lambda,L}^\star - I(\theta; X_1^n) \Big{]},
\end{align}
where $I(\theta; X_1^n)$ is the mutual information, and 
\begin{equation}
    \rho_{\lambda,L}^\star \coloneqq - \sup_a \ \log \int_{\Theta} \exp \Big{(}-\lambda L(\theta, a) \Big{)} \ p(d\theta).
\end{equation}
\end{restatable}
\begin{proof}
The proof is in App.~\ref{app:proofs}. In contrast to \cite{Chen16, Gerchinovitz20,Majumdar22}, we use the Donsker-Varadhan change of measure inequality \citep[Thm. 2.3.2]{Gray11} to handle unbounded losses. 
\end{proof}
Consistent with intuition, this bound increases when the data $X_1^n$ do not provide much information about the underlying $\theta$ (i.e., when the mutual information is small). The term $\rho_{\lambda,L}^\star$ is related to the best achievable (exponentiated) average loss when one simply chooses an action $a$ \emph{without} observing any data. This term may be easily computed when $\A$ and $\Theta$ are finite; here, the supremum over $a$ and the expectation can both be computed exactly. In addition, one may also compute a bound when $\A$ is finite, but $\Theta$ is not (by computing or bounding the expectation above with high probability via sampling and applying a concentration inequality; the supremum over $a$ can then be computed by enumeration). We also note that $\rho_{\lambda,L}^\star$ is similar to the ``small-ball probability" that appears in previous Fano inequalities \citet{Chen16, Gerchinovitz20}. While there is no general recipe for bounding the small-ball probability, this can be achieved on a case-by-case basis for particular examples (see \citet{Chen16}). The quantity $\rho_{\lambda,L}^\star$ similarly needs to be computed on a case-by-case basis in general.

\begin{corollary}[Prioritized risk lower bound via generalized Fano inequality]
\label{corollary: generalized fano}
For any distribution $p$ on $\Theta$, we have the following lower bound for all $\lambda > 0$:
\begin{align}
    \R_\text{prior}(\pi, L; \Theta) \geq \frac{1}{\lambda} \Big{[} \rho^\star_{\lambda, L^\pi} - I(\theta; X_1^n) \Big{]},
\end{align}
where $I(\theta; X_1^n)$ is the mutual information (computed using $p$ for $\theta$), and 
\begin{equation}
\rho^\star_{\lambda, L^\pi} \coloneqq - \sup_a \ \log \int_{\Theta} \exp \Big{(}-\lambda \pi(\theta) L(\theta, a) \Big{)} \ p(d\theta). \nonumber
\end{equation}
\end{corollary}
\begin{proof}
    The result follows directly by combining Thm.~\ref{thm:generalized fano} with the following (which follows from the fact that a supremum is lower bounded by an average): $\R_\text{prior}(\pi, L; \Theta) \geq \R_\text{Bayes}(p, L^\pi; \Theta)$,
    where $L^\pi(\theta, a) \coloneqq \pi(\theta) L(\theta, a)$.
\end{proof}

We note that the bound above holds for any choice of $p$ and $\lambda$; thus, we are free to choose $p$ and $\lambda$ judiciously in order to obtain a lower bound on the prioritized risk.

\section{Examples}
\label{sec:examples}

We illustrate the ability of the prioritized risk framework to provide insights into tradeoffs between prior information and learning performance for prior-informed versions of: (i) Bernoulli mean estimation, (ii) logistic regression, and (iii) RL with environments drawn from Zipfian distributions. 

\subsection{Bernoulli mean estimation with priors}

We start by considering the problem of estimating the mean of a Bernoulli distribution in the presence of prior information. 
Suppose we have a family of Bernoulli distributions $P_\theta$ over $\X \in \{0, 1\}$; the distributions are parameterized using the mean $\theta \in \Theta \coloneqq [0,1]$. Suppose a learner has a prior $\pi$ over $\theta$. The learner's goal is to estimate $\theta$ given data from $P_\theta$. The learner $\sigma: \X^n \rightarrow \Theta$ outputs an estimate $\hat{\theta} \coloneqq \sigma(x_1^n)$ and is evaluated using the loss $L(\theta, \hat{\theta}) \coloneqq |\theta - \hat{\theta}|$.

In order to obtain numerical results, we choose a prior of a particular form. Let $\pi$ correspond to the density function of a Beta distribution Beta($\alpha=1$, $\beta=2$); this prior assigns higher weight to low values of the mean $\theta$ (specifically, $\pi$ is linear in $\theta$ with $\pi(0) = 2$ and $\pi(1) = 0$). We can then apply LeCam's method for lower bounding the prioritized risk (Thm.~\ref{thm:lecam}). By applying Pinsker's inequality and the tensorization property of the KL divergence, we obtain:
\begin{equation}
    \R_\text{prior}(\pi, L; \Theta) \geq \frac{\delta}{2} \Bigg{[} 1 - \sqrt{\frac{n \text{KL}(P_{\theta_0} \| P_{\theta_1})}{2}} \Bigg{]},
\end{equation}
where $\{\theta_0, \theta_1\}$ forms a $(\delta, \pi)$-packing. Since the KL divergence can be computed analytically for Bernoulli distributions, we can then find a value of $\delta$ that maximizes this bound for each $n$. 

Fig.~\ref{fig:bernoulli} plots the lower bound on the prioritized risk $\R_\text{prior}$ as a function of the number of samples $n$. As described in Sec.~\ref{sec:implications for learning problems}, this establishes a fundamental tradeoff for learners for this problem (for each value of $n$). We also compare the lower bound for the non-uniform prior Beta($\alpha=1$, $\beta=2$) with the lower bound for a uniform prior ($\pi(\theta) = 1, \forall \theta \in [0,1]$), which corresponds to the minimax risk.

In Appendix~\ref{app:upper bounds}, we present \emph{upper bounds} on prioritized risk for this problem computed using (i) Bayesian inference with prior $\pi$ and (ii) a learner that does not exploit the prior (Bayesian inference with uniform prior). As expected, Bayesian inference with $\pi$ achieves a lower learner-specific prioritized risk. However, we also demonstrate that Bayesian inference with $\pi$ is not optimal in general from the perspective of the prioritized risk; we do this by constructing a different learner that achieves a lower learner-specific prioritized risk. This thus motivates the search for learning algorithms that achieve optimal prioritized risk. 



 \begin{figure}[t]
 \begin{center}
 \includegraphics[width=0.99\columnwidth]{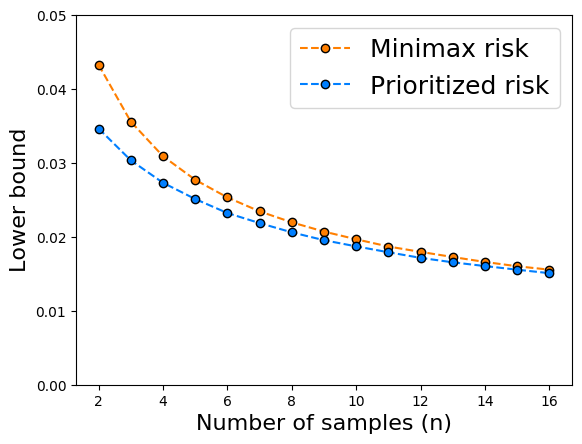}
 \end{center} 
 \vspace{-15pt}
 \caption{Lower bounds on $\R_\text{prior}(\pi, L; \Theta)$ and $\R_\text{minimax}(L; \Theta)$ for Bernoulli mean estimation.  \label{fig:bernoulli}}
 \vspace{-10pt}
 \end{figure}

 \subsection{Logistic regression with directional priors}
 \label{sec:logistic regression}

 We now consider the problem of logistic regression with priors, and utilize Assouad's method to obtain lower bounds on the prioritized risk. Given a fixed set of regressors $\{z_i\}_{i=1}^n$, where $z_i \in \RR^d$, the logistic regression model assigns the following probability to the label $y \in \{-1,1\}$:
 \begin{equation}
 \label{eq:logistic model}
     P(Y_i=y | z_i, \theta) = \frac{1}{1 + \exp(-y z_i^T \theta)}.
 \end{equation}
 The goal of the learner is to infer the unknown parameter $\theta \in \RR^d$ given $n$ observations of labels $\{y_i\}_{i=1}^n$. Here, we use the $L_1$-error $L(\theta, \hat{\theta}) \coloneqq \|\theta - \hat{\theta}\|_1$ to evaluate the learner. 

The learner has a prior $\pi$ on the parameter $\theta$. Let $\V \coloneqq \{-1,1\}^d$ denote the vertices of the hybercube, and suppose that $\pi$ is normalized such that $\pi(\theta_{v_j}) + \pi(\theta_{v_{-j}}) = 1$. Define:
\begin{equation}
    \lambda_j \coloneqq \frac{1}{4} \Bigg{[}\frac{1}{\pi(\theta_{v_j})} + \frac{1}{\pi(\theta_{v_{-j}})} \Bigg{]}. 
\end{equation}
This quantity captures \emph{directional} biases imposed by the prior. Specifically, given the logistic model \eqref{eq:logistic model}, we see that if $\pi(\theta_{v_{j}})$ is higher than $\pi(\theta_{v_{-j}})$, this implies confidence that the $j$-th component of $z$ being positive will lead to the label $y$ being positive (and that the $j$-th component of $z$ being negative will lead to the label $y$ being negative). Such asymmetries are captured by $\lambda$. Specifically, if the prior is symmetric such that $\pi(\theta_{v_{j}}) = \pi(\theta_{v_{-j}}) = 0.5$, then $\lambda_j = 1$. If the prior is asymmetric, then $\lambda_j > 1$ (with the distance from 1 capturing the degree of asymmetry). 

Defining $[\hat{v}(\theta)]_j \coloneqq \text{sign}(\theta_j)$, we see that $\V$ forms a $(\delta,\pi)$-Hamming separation (Defn.~\ref{defn:hamming separation}). We can thus apply Assouad's method (Thm.~\ref{thm:assouad}) to lower bound the prioritized risk  (see Appendix~\ref{app:proofs} for the proof):
\begin{equation}
\label{eq:logistic bound}
    \R_\text{prior}(\pi, L; \Theta) \geq \frac{1}{16} \cdot \frac{d^\frac{3}{2}}{\sqrt{\sum_{j=1}^d \sum_{i=1}^n \lambda_j^2 z_{ij}^2}}.
\end{equation}
As a special case, if $\lambda_j \eqqcolon \lambda \ (\forall j$), we obtain:
\begin{equation}
\label{eq:logistic bound simplified}
    \R_\text{prior}(\pi, L; \Theta) \geq \frac{1}{16} \cdot \frac{d^\frac{3}{2}}{\lambda \|Z\|_\text{Fr}},
\end{equation}
where $\|Z\|_\text{Fr}$ is the Frobenius norm of the matrix $Z \in \RR^{d \times n}$ consisting of the regressors $z_i$. 

We observe that for a fixed dimension $d$, the bound \eqref{eq:logistic bound simplified} on $\R_\text{prior}$ is determined by the product between $\lambda$ and $\|Z\|_\text{Fr}$ (and, similarly for \eqref{eq:logistic bound}, the bound depends on the products $\lambda_j z_{ij}$). Thus, for a given prior $\pi$, we see that regressors with a smaller $L_2$-norm induce a \emph{more stringent} tradeoff between prior information and risk of a learner. As a concrete example, we build on the analysis in Example~\ref{sec:implications for learning problems}. Consider two learning problems corresponding to two sets of regressors with $\|Z\|_\text{Fr} \leq \|Z'\|_\text{Fr}$. Suppose we have a learner $\sigma'$ that achieves low risk for high values of the prior using $Z'$:
\begin{equation}
\R(\sigma', \theta; Z') < \beta' \coloneqq \frac{1}{16} \cdot \frac{d^\frac{3}{2}}{\lambda \|Z'\|_\text{Fr}}, \ \forall \theta \ \text{s.t.} \ \frac{1}{2} \leq \pi(\theta) \leq 1. \nonumber
\end{equation}
Now, suppose we have a learner $\sigma$ that achieves the \emph{same} level of performance as $\sigma'$ for high values of the prior using regressors $Z$:
\begin{align}
\R(\sigma, \theta; Z) < \beta' &\coloneqq \frac{1}{16} \cdot \frac{d^\frac{3}{2}}{\lambda \|Z'\|_\text{Fr}}, \ \forall \theta \ \text{s.t.} \ \frac{1}{2} \leq \pi(\theta) \leq 1 \nonumber \\
&\leq \frac{1}{16} \cdot \frac{d^\frac{3}{2}}{\lambda \|Z\|_\text{Fr}}.
\end{align}
Then, building on Example~\ref{sec:implications for learning problems}, we see that for the problem with regressors $Z'$, 
\begin{equation}
    \exists \theta \ \text{s.t.} \ \pi(\theta) < \frac{1}{2}, \ \text{where} \ \R(\sigma', \theta; Z') > \frac{1}{8} \cdot \frac{d^\frac{3}{2}}{\lambda \|Z'\|_\text{Fr}},
\end{equation}
whereas using regressors $Z$, we have:
\begin{align}
    \exists \theta \ \text{s.t.} \ \pi(\theta) < \frac{1}{2}, \ \text{where} \ \R(\sigma, \theta; Z) &> \frac{1}{8} \cdot \frac{d^\frac{3}{2}}{\lambda \|Z\|_\text{Fr}} \\
    &\geq \frac{1}{8} \cdot \frac{d^\frac{3}{2}}{\lambda \|Z'\|_\text{Fr}}. \nonumber
\end{align}
The bounds on prioritized risk thus suggest that in order for $\sigma$ to achieve the same level of performance (using $Z$) as $\sigma'$ (using $Z'$) for high values of the prior, $\sigma$ must sacrifice a \emph{greater} level of performance for low prior values.  


\subsection{Reinforcement Learning in Zipfian environments}

In our final example, we consider a reinforcement learning (RL) setting where an agent interacts with environments that are drawn from a Zipfian (i.e., discrete power law) distribution. Specifically, we build on the Zipfian Gridworld environments from \citet{Chan22}, where an agent must find objects using visual feedback. Each environment (Fig.~\ref{fig:zipf} left) consists of a grid-world with four rooms containing 20 objects. There are a fixed set of 400 environments that the agent may be deployed in; the agent's start location, target object, as well as the other object shapes, colors, and locations are fixed within each environment. In each episode, the agent receives a top-down camera view of its immediate surroundings along with a visual depiction of the target object. The agent receives a loss of 0 for the episode if it reaches the target; the episode ends if the agent touches any other object, and a loss equal to the number of steps taken is then assigned. 

 \begin{figure}[t]
 \begin{center}
 \includegraphics[width=0.40\columnwidth]{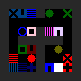}
\includegraphics[width=0.54\columnwidth]{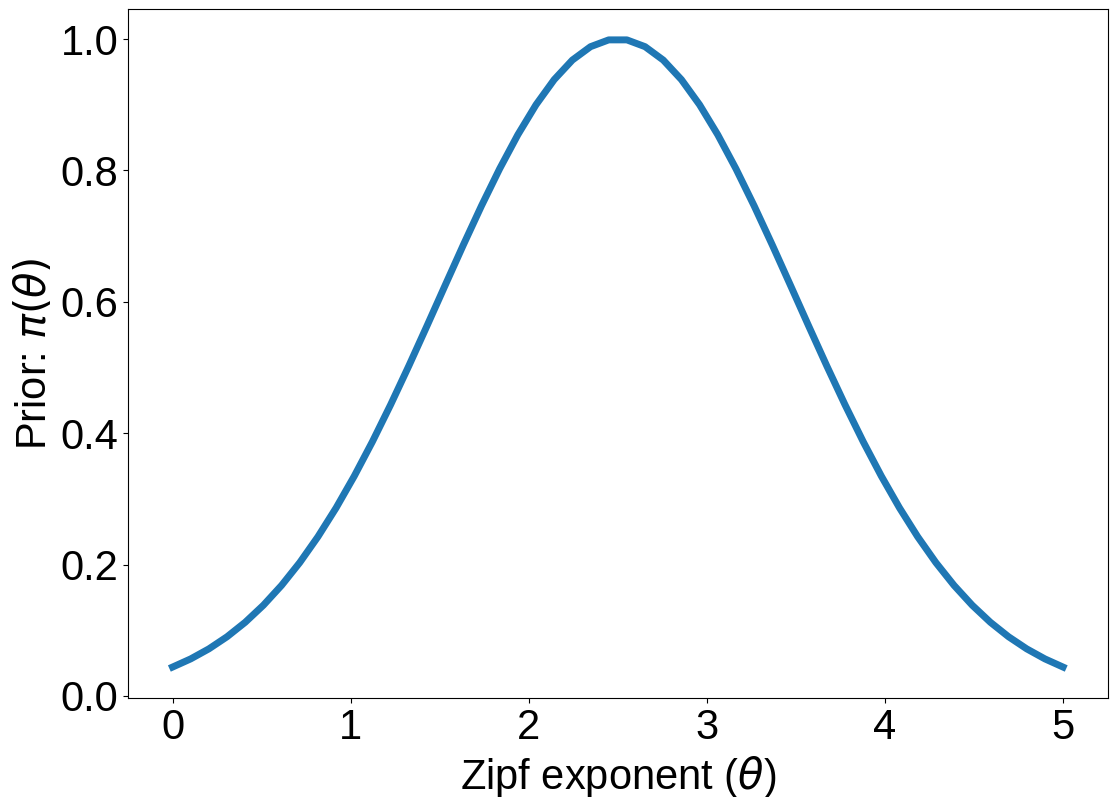}
 \end{center} 
 \vspace{-10pt}
 \caption{L: A Zipfian Gridworld environment. R: Prior $\pi(\theta)$.  \label{fig:zipf}}
 \end{figure}

\begin{figure}[t]
 \begin{center}
 \includegraphics[width=0.8\columnwidth]{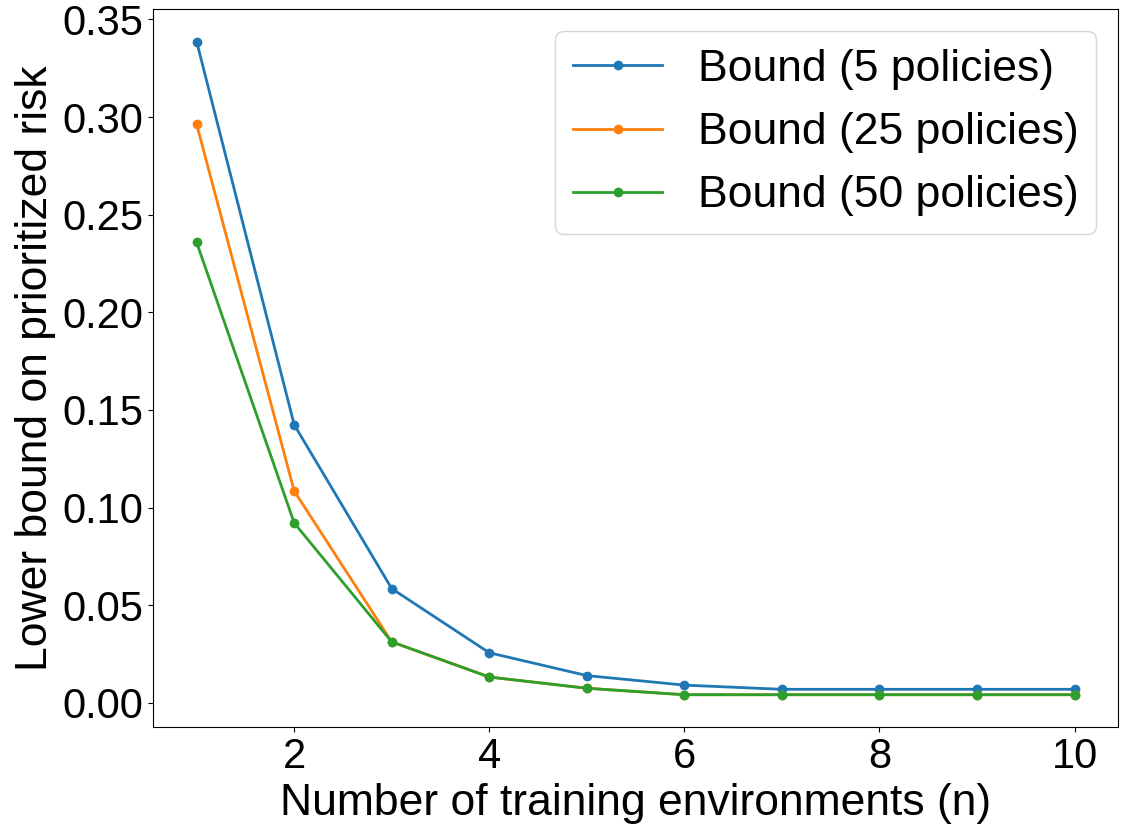}
 \end{center} 
 \vspace{-10pt}
 \caption{Lower bounds on prioritized risk (Zipfian environments).  \label{fig:zipf bounds}}
 \vspace{-25pt}
 \end{figure}

Let $\X = \{e_x\}_{x=1}^{400}$ denote the set of possible environments. The distribution $P_\theta$ over environments is defined by a discrete power law (Zipfian distribution) $p(x) \propto \frac{1}{x^\theta}$,
where the exponent $\theta$ determines how skewed the distribution is. Such a distribution captures the heavy-tailed nature of many real-world environments \cite{Chan22}. In our setting, $\theta$ is chosen from a set $\Theta$ of 50 possible exponents between $0$ and $5$. Here, we use the prioritized risk to study fundamental limits on the agent's learning performance when $\theta$ (i.e., how heavy-tailed the distribution is) does not conform to the learner's prior. The agent does not have \emph{a priori} knowledge of $\theta$, but has a prior $\pi(\theta) = \exp(-(\theta - 2.5)^2)$, which captures the prior knowledge that real-world distributions are likely to have Zipf exponents close to 2.5 (Fig.~\ref{fig:zipf} right). For training, the agent receives $n$ environments from $P_\theta$. Based on this training set, it must choose a policy (i.e., a mapping from images to actions) from a set $\A$. Here, we choose a (pre-computed) set $\A$ of policies, each trained on a given Zipfian exponent $\theta$ using IMPALA \cite{Espeholt18}.  

We use the generalized Fano inequality to compute lower bounds on the prioritized risk (Cor.~\ref{corollary: generalized fano}); since $\Theta$ and $\A$ are finite, we can compute all quantities in the bound exactly. Fig.~\ref{fig:zipf bounds} plots the resulting bounds for different sizes of $\A$. We see that a learning agent that has access to a smaller set $\A$ faces a \emph{more stringent} tradeoff between prior information and expected cost. Analogous to the analysis in Sec.~\ref{sec:logistic regression}, the bound suggests that for an agent $\sigma_{|\A| = 5}$ (with access to $\A$ of size 5) to achieve the same level of performance as $\sigma_{|\A| = 50}$ (with access to $\A$ of larger size) for high values of the prior (i.e., $\theta$ close to 2.5), $\sigma_{|\A| = 5}$ must sacrifice a \emph{greater} level of performance for less likely values of $\theta$.

\section{Related Work}
\label{sec:related work}


{\bf No-free-lunch theorems and minimax lower bounds.} The no-free-lunch (NFL) theorem \cite{Wolpert96}\citep[Ch. 5]{Shalev14} establishes fundamental limits on learners that have no prior information. 
One statement of the NFL theorem is via the minimax risk: if one considers binary classification tasks and allows for any distribution whatsoever over samples, then the minimax is lower bounded by a constant (for all sample sizes $n$). In general, lower bounds on the minimax risk for a given learning problem allow us to formalize the fundamental limits of learning on that problem. Minimax lower bounds have been established for many learning problems of practical interest, e.g., sparse linear regression \cite{Raskutti11}, nonparametric classification \cite{Yang99a}, crowdsourcing \cite{Zhang14}, differentially private learning \cite{Duchi13}, and inverse reinforcement learning \cite{Komanduru21}.  Such lower bounds are typically proved using information-theoretic techniques, e.g., LeCam's method, Fano's inequality, or Assouad's method \cite{Tsybakov08, Yang99, Yu97}. As highlighted in Sec.~\ref{sec:background}, the minimax risk does not reason about prior knowledge beyond the weak knowledge that the data-generating distribution belongs to a certain set. 


{\bf Lower bounds on Bayes risk.} The Bayes risk (Eq. \eqref{eq:bayes risk}) considers the average risk relative to a prior $\pi$ over data-generating distributions (i.e., over $\theta)$. 
Bayes risk lower bounds are known for a variety of problems including  finite-dimensional estimation problems with quadratic losses \cite{VanTrees04, Brown90, Gill95}, estimation problems involving generalized linear models \cite{Chen16}, and high-dimensional sparse linear regression \cite{Chen16}. As highlighted in Sec.~\ref{sec:background}, the Bayes risk framework assumes that the learner has access to the ``true" distribution over $\theta$, and thus does not capture settings where reality does not conform to the prior. 

More refined versions of the Bayes risk include ``local" versions \cite{Zhang06} where one considers a partition of the parameter space $\Theta$ into disjoint balls (of small size) and computes a lower bound on the Bayes risk defined relative to a local prior. 
This quantity is distinct from the one we consider here. 
The work presented by \citet{Haussler94} considers the tradeoff between prior knowledge and sample complexity by analyzing the impact of a misspecified prior on the Bayes risk. However, \citet{Haussler94} only provide upper bounds on sample complexity in the setting with misspecified priors; lower bounds are only provided for the case where the learner has access to nature's prior. More recently, \citet{Chollet19} considers a formalism based on algorithmic information theory for capturing tradeoffs between prior knowledge and learning efficiency (but computing the relevant quantities is computationally intractable in general).



{\bf PAC-Bayes generalization bounds.}  
One motivation for studying fundamental limits on learners with prior information comes from PAC-Bayes theory \cite{McAllester99}, which provides \emph{upper bounds} on risk for learners equipped with prior information. In PAC-Bayes learning, one defines a prior distribution over the space of hypotheses, and then obtains a bound on the risk that holds for any choice of posterior. This bound can be operationalized by fixing a data-independent prior and then finding a data-dependent posterior that minimizes the PAC-Bayes bound \cite{Dziugiate17}. 
The idea of using such prior information to improve generalization bounds is also a key component of the ``luckiness" framework \cite{Shawe-Taylor98}, Occam's razor bounds \cite{Blumer87, Langford05}, and the minimum description length principle \cite{Rissanen89}. PAC-Bayes provides some of the tightest known generalization bounds for neural networks for supervised learning \cite{Dziugiate17, Neyshabur17, Neyshabur17a, Bartlett17, Arora18, Rivasplata19, Perez-Ortiz20, Jiang20, Lotfi22} and policy learning \cite{Fard12, Majumdar21, Veer20, Ren21}. 

These results provide a motivation for the question we consider here: how can we establish fundamental limits on learners with imperfect prior information? Empirically, one often observes that the strength of PAC-Bayes bounds depends significantly on how good the prior is. If the data-generating distribution conforms to the PAC-Bayes prior (e.g., if the prior achieves low expected loss on the data-generating distribution), then the resulting PAC-Bayes bound can be strong (i.e., the upper bound on the expected loss can be low). However, if the prior is not well-matched with the data-generating distribution, the posterior may have to deviate significantly from the prior, resulting in a poor upper bound. Thus, one observes a tradeoff with PAC-Bayes bounds: if the data-generating distribution conforms to the learner’s prior, one can achieve low risk; however, if the data-generating distribution does not conform to the learner’s prior, the PAC-Bayes bound on the risk is high.

This tradeoff is very similar to the one the prioritized risk attempts to capture. However, while the PAC-Bayes framework provides upper bounds on risk, we focus on lower bounds. We are thus motivated by the goal of establishing fundamental tradeoffs (i.e., ones that hold for any learning algorithm) between the accuracy of prior information and learning performance, while the PAC-Bayes approach provides a \emph{particular} such tradeoff. An interesting question for future work is thus the following: does the PAC-Bayes framework (or a variant of it) provide the optimal tradeoff (i.e., does the upper bound from PAC-Bayes match lower bounds on the prioritized risk) for any classes of learning problems? If not, can we find algorithms that are optimal from the perspective of the prioritized risk?


\section{Discussion and conclusions}

We have introduced the notion of prioritized risk, which differs from classical notions of minimax and Bayes risk by allowing us to study fundamental tradeoffs in settings where reality does not conform to the learner's prior. Specifically, lower bounds on the prioritized risk for a given problem establish that it is \emph{impossible} for both the risk of a learner and the prior to be low for all distributions. We have extended classical techniques based on the methods of LeCam, Assouad, and Fano for obtaining lower bounds on the prioritized risk for estimation. We also presented a technique for obtaining lower bounds in more general settings via a novel generalized Fano inequality (which may be of independent interest for lower bounding Bayes risk in settings with unbounded loss functions). 


\subsection{Future work}

There are a number of exciting directions for future work. First, developing prior-informed learning algorithms that are optimal from the perspective of the prioritized risk would be of practical interest (e.g., based on variants of PAC-Bayes bounds). Empirically, one often observes tradeoffs between accuracy of prior information and learning performance: for some problems, the choice of neural network architecture or regularization technique (which can both be interpreted as forms of inductive bias / prior knowledge) seem to have little impact on learning performance, while for other problems these choices can have significant impact. 
This raises the following important question: do the empirical observations reflect fundamental tradeoffs (as formalized by the prioritized risk), or are they artifacts of the specific learning algorithms we happen to be using? If the latter, this motivates the search for different learning algorithms for our problems of interest. 


Second, an interesting theoretical direction is to explore if there are settings where the optimal asymptotic dependence on $n$ (number of examples) for prioritized risk and minimax/Bayes risk are different (analogous to the difference between universal learning \cite{Bousquet21} and PAC learning). To expand on this, consider a learner $\sigma_\text{minimax}$ that achieves the optimal asymptotic rate for the minimax risk. There are then two possibilities. First, it may be the case that $\sigma_\text{minimax}$ does not achieve the optimal asymptotic rate for the prioritized risk. The practical implication is that one should use a different learner based on whether or not one has a uniform or non-uniform prior (even for large amounts of data). Second, it may be the case that $\sigma_\text{minimax}$ does achieve the optimal asymptotic rate for the prioritized risk. This implies that the same learner achieves a different asymptotic tradeoff between prior knowledge and learning performance depending on whether or not one has a uniform prior. Each possibility is interesting in its own right, and provides insights into the tradeoffs between prior knowledge and learning performance for the problem under consideration.

Finally, we are interested in using the prioritized risk framework to understand how much prior information (``nature") one needs in order to achieve a certain level of learning performance (``nurture") in a broader set of applications of practical interest (e.g., RL in robotics). 

\section*{Acknowledgements}
The author would like to thank Olivier Bousquet for a helpful discussion on this work, and Asher Hancock, Nate Simon, David Snyder, and Vince Pacelli for providing feedback on an early draft. This work was partially supported by the NSF CAREER Award [\#2044149] and the Office of Naval Research [N00014-23-1-2148].


\balance
\bibliography{lower-bounds}
\bibliographystyle{icml2023}

\newpage
\appendix
\onecolumn
\section{Proofs}
\label{app:proofs}

\estimationreduction*
\begin{proof}
Let $\sigma: \X^n \rightarrow \Theta$ be an estimator. 
We first observe:
\begin{align}
\R_\text{prior}^\sigma(\pi, L; \Theta) &= \sup_{\theta \in \Theta} \ \pi(\theta) \underset{X_1^n \sim P_\theta^n}{\mathbb{E}} \big{[} \rho(\theta, \sigma(X_1^n)) \big{]} \\
&= \sup_{\theta \in \Theta} \ \underset{X_1^n \sim P_\theta^n}{\mathbb{E}} \big{[} \pi(\theta) \rho(\theta, \sigma(X_1^n)) \big{]} \\
&\geq \sup_{\theta \in \Theta} \ \underset{X_1^n \sim P_\theta^n}{\mathbb{E}} \big{[} \delta \cdot \mathbbm{1} \Big{[} \pi(\theta) \rho(\theta, \sigma(X_1^n)) \geq \delta \big{]} \Big{]} \\
&= \sup_{\theta \in \Theta} \ \delta \cdot P_\theta \big{[} \pi(\theta) \rho(\theta, \sigma(X_1^n)) \geq \delta \big{]}. \label{eq:prob-lower-bound}
\end{align}
Now consider a family of distributions $\{\theta_v \}_{v \in \V}$ that forms a $(\delta, \pi)$-packing, and define the prior-weighted test function:
\begin{equation}
\Psi(x_1^n) \coloneqq \underset{v \in \V}{\text{argmin}} \ \pi(\theta_v) \rho(\theta_v, \sigma(x_1^n)).  
\end{equation}
Now suppose that $\pi(\theta_v) \rho(\theta_v, \sigma(x_1^n)) < \delta$. Then, we claim that $\psi(x_1^n) = v$. To see this, suppose not. Then, 
\begin{align}
\exists v' \ \text{s.t.} \  \ &\pi(\theta_{v'}) \rho(\theta_{v'}, \sigma(x_1^n)) < \pi(\theta_v) \rho(\theta_{v}, \sigma(x_1^n)) \\
\implies \ \ &\rho(\theta_{v'}, \sigma(x_1^n)) < \frac{\pi(\theta_v) \rho(\theta_{v}, \sigma(x_1^n))}{\pi(\theta_{v'})} < \frac{\delta}{\pi(\theta_{v'})}.
\end{align}
Thus, $\sigma(x_1^n)$ is in the $\rho$-ball of radius $\delta/\pi(\theta_{v'})$ around $\theta_{v'}$. But, we said that $\sigma(x_1^n)$ is in the $\rho$-ball of radius $\delta/\pi(\theta_{v})$ around $\theta_{v}$, and that the balls are non-overlapping (since we have a $(\delta, \pi)$-packing). Thus,  $\psi(x_1^n) = v$. 

Considering the contrapositive, we see that if $\psi(x_1^n) \neq v$, then $\rho(\theta_v, \sigma(x_1^n)) \geq \delta / \pi(\theta_v)$. Averaging over $\V$, we see:
\begin{align}
    \sup_\theta \ \mathbb{P} \Big{[} \rho(\theta, \sigma(X_1^n)) \geq \frac{\delta}{\pi(\theta)} \Big{]} &\geq \frac{1}{|\V|}  \sum_{v \in \V} \mathbb{P} \Big{[} \rho(\theta_v, \sigma(X_1^n)) \geq \frac{\delta}{\pi(\theta_v)} \ \Big{|} \ V=v \Big{]} \\
    &\geq \frac{1}{|\V|}  \sum_{v \in \V} \mathbb{P} \Big{[} \Psi(X_1^n) \neq v \ \Big{|} \ V=v \Big{]} \\
    &\eqqcolon \mathbb{P}(\Psi(X_1^n) \neq V).
\end{align}
Combining this with \eqref{eq:prob-lower-bound} establishes the claim in the proposition. 

\end{proof}

\lecam*
\begin{proof}
The proof follows directly by combining Proposition~\ref{prop:reduction} with the well-known variational representation of the total variation distance (see, e.g., \citep[Proposition 2.17]{Duchi16}): for distributions $P_0$ and $P_1$ on sample space $\X$, we have
\begin{equation}
    \inf_\Psi \ \{P_0(\Psi(X) \neq 0) + P_1(\Psi(X) \neq 1)\} = 1 - \|P_0 - P_1\|_\text{TV}, \nonumber
\end{equation}
where the infimum is over test functions $\Psi: \X \rightarrow \{0,1\}$. 
\end{proof}

\fano*
\begin{proof}
    The proof follows directly by combining Proposition~\ref{prop:reduction} with Fano's inequality \cite{Cover12}:
    \begin{equation}
        \inf_{\Psi} \ \mathbb{P} (\Psi (X_1^n) \neq V) \geq 1 - \frac{I(V;X_1^n) + \log(2)}{\log |\V|}.
    \end{equation}
\end{proof}

\assouad*
\begin{proof}
Fix an arbitrary estimator $\sigma: \X^n \rightarrow \Theta$. We then have:
\begin{align}
    \R_\text{prior}^\sigma(\pi, L; \Theta) &= \sup_{\theta \in \Theta} \ \pi(\theta) \underset{X_1^n \sim P_\theta^n}{\mathbb{E}} \big{[} \rho(\theta, \sigma(X_1^n)) \big{]} \\
    &\geq \frac{1}{|\V|} \sum_{v \in \V} \pi(\theta_v) \underset{X_1^n \sim P_{\theta_v}^n}{\mathbb{E}} \big{[} \rho(\theta_v, \sigma(X_1^n)) \big{]} \\
    &\geq \frac{1}{|\V|} \sum_{v \in \V} \underset{X_1^n \sim P_{\theta_v}^n}{\mathbb{E}} \Big{[} 2\delta \sum_{j=1}^d \mathbbm{1} \big{\{} [\hat{v}(\sigma(X_1^n)]_j \neq v_j \big{\}} \Big{]} \ \ \ (\text{By $(2\delta,\pi)$-Hamming separation.}) \\
    &= \sum_{j=1}^d \frac{2\delta}{|\V|} \sum_{v \in \V} P_{\theta_v} \Big{(} [\hat{v}(\sigma(X_1^n)]_j \neq v_j \Big{)} \\
    &= \sum_{j=1}^d \frac{2\delta}{|\V|} \Bigg{[} \sum_{v | v_j=+1}  P_{\theta_v} \Big{(} [\hat{v}(\sigma(X_1^n)]_j \neq v_j \Big{)}  + \sum_{v | v_j=-1}  P_{\theta_v} \Big{(} [\hat{v}(\sigma(X_1^n)]_j \neq v_j \Big{)} \Bigg{]} \\
    &= \delta \sum_{j=1}^d \Bigg{[} \mathbb{P}_{+j} \Big{(} [\hat{v}(\sigma(X_1^n)]_j \neq v_j \Big{)} + \mathbb{P}_{-j} \Big{(} [\hat{v}(\sigma(X_1^n)]_j \neq v_j \Big{)} \Bigg{]}.
\end{align}
Taking an infimum over estimators (on the LHS) and test functions $\Psi: \X^n \rightarrow \{+1, -1\}$ (on the RHS) establishes the desired result. 

\end{proof}

\generalizedfano*
\begin{proof}
    The Donsker-Varadhan change of measure inequality \citep[Theorem 2.3.2]{Gray11} states that for any random variable $Z$, we have the following inequality for all distributions $P$ and $Q$:
    \begin{equation}
    \underset{Z \sim P}{\mathbb{E}} [Z] \leq \text{KL}(P \| Q) + \log \underset{Z \sim Q}{\mathbb{E}} \exp (Z).
    \end{equation}
    Choosing $Z = -\lambda L(\theta, \sigma(X_1^n))$, we then have:
    \begin{align}
    \underset{p(\theta, X_1^n)}{\mathbb{E}} \lambda L(\theta, \sigma(X_1^n)) \geq -\text{KL}(p(\theta, X_1^n) \| p(\theta) q(X_1^n)) - \log \underset{p(\theta) q(X_1^n)}{\mathbb{E}} \exp \Big{[} -\lambda L(\theta, \sigma(X_1^n)) \Big{]},
    \end{align}
    where $p(\theta, X_1^n)$ is the joint distribution defined by $p$ and $P_\theta$, and $q$ is any arbitrary distribution on $X_1^n$. 
    Taking an infimum over $\sigma$ on both sides, we have:
    \begin{align}
    \inf_\sigma \ \underset{p(\theta, X_1^n)}{\mathbb{E}} \lambda L(\theta, \sigma(X_1^n)) &\geq -\text{KL}(p(\theta, X_1^n) \| p(\theta) q(X_1^n)) - \sup_\sigma \  \log \underset{p(\theta) q(X_1^n)}{\mathbb{E}} \exp \Big{[} - \lambda L(\theta, \sigma(X_1^n)) \Big{]} \\
    &= -\text{KL}(p(\theta, X_1^n) \| p(\theta) q(X_1^n)) -  \log \Bigg{[} \sup_\sigma \ \underset{p(\theta) q(X_1^n)}{\mathbb{E}} \exp \Big{[} - \lambda L(\theta, \sigma(X_1^n)) \Big{]} \Bigg{]}.
    \end{align}
    The equality above follows from the monotonicity of the log function. Now, via the Fubini-Tonelli theorem, we have:
    \begin{align}
        \log \Bigg{[} \sup_\sigma \ \underset{p(\theta)}{\mathbb{E}} \underset{q(X_1^n)}{\mathbb{E}} \exp \Big{[} -\lambda L(\theta, \sigma(X_1^n)) \Big{]} \Bigg{]} &= \log \Bigg{[} \sup_\sigma \ \underset{q(X_1^n)}{\mathbb{E}} \underset{p(\theta)}{\mathbb{E}} \exp \Big{[} - \lambda L(\theta, \sigma(X_1^n)) \Big{]} \Bigg{]} \\
        &\leq \log \Bigg{[} \underset{q(X_1^n)}{\mathbb{E}} \underset{p(\theta)}{\mathbb{E}} \sup_\sigma \ \exp \Big{[} -\lambda L(\theta, \sigma(X_1^n)) \Big{]} \Bigg{]} \\
        &= \log \Bigg{[} \underset{q(X_1^n)}{\mathbb{E}} \underset{p(\theta)}{\mathbb{E}} \sup_a \ \exp \Big{[} -\lambda L(\theta, a) \Big{]} \Bigg{]} \\
        &= \sup_a \ \log \Bigg{[} \underset{p(\theta)}{\mathbb{E}} \ \exp \Big{[} - \lambda L(\theta, a) \Big{]} \Bigg{]} \\
        &\eqqcolon -\rho^\star_{\lambda, L}.
    \end{align}
    We thus have:
    \begin{align}
        \inf_\sigma \ \underset{p(\theta, X_1^n)}{\mathbb{E}} \lambda L(\theta, \sigma(X_1^n)) \geq -\text{KL}(p(\theta, X_1^n) \| p(\theta) q(X_1^n)) + \rho^\star_{\lambda, L}.
    \end{align}
    Noting that this inequality holds for any choice of $q$, we can supremize over $q$ to obtain the tightest bound:
    \begin{align}
        \sup_q \ -\text{KL}(p(\theta, X_1^n) \| p(\theta) q(X_1^n)) = -\inf_q \ \text{KL}(p(\theta, X_1^n) \| p(\theta) q(X_1^n)) = -I(\theta; X_1^n).
    \end{align}
    We thus obtain the desired result:
    \begin{align}
        \R_\text{Bayes}(p, L; \Theta) &= \inf_\sigma \ \underset{p(\theta, X_1^n)}{\mathbb{E}} L(\theta, \sigma(X_1^n)) \\
        &\geq \frac{1}{\lambda} \Big{[} \rho^\star_{\lambda, L} - I(\theta; X_1^n) \Big{]}. 
    \end{align}

\end{proof}


{\bf Lower bound for logistic regression (Sec.~\ref{sec:logistic regression}).} 
We prove the following lower bound for the logistic regression problem described in Sec.~\ref{sec:logistic regression}:
\begin{equation}
    \R_\text{prior}(\pi, L; \Theta) \geq \frac{1}{16} \cdot \frac{d^\frac{3}{2}}{\sqrt{\sum_{j=1}^d \sum_{i=1}^n \lambda_j^2 z_{ij}^2}}.
\end{equation}
\begin{proof}
    Using Assouad's method (Thm.~\ref{thm:assouad}) and \eqref{eq:assouad tv} we see:
    \begin{align}
        \R_\text{prior}(\pi, L; \Theta) &\geq \frac{\delta}{2} \sum_{j=1}^d \Big{[} 1 - \big{\|} P_{{+j}}^n - P_{{-j}}^n \big{\|}_\text{TV} \Big{]} \\
        &\geq \frac{d \delta}{2} \Bigg{[} 1 - \Bigg{(} \frac{1}{d} \sum_{j=1}^d \frac{1}{2^d} \sum_{v \in \V} \| P_{v,+j} - P_{v,-j} \|^2_\text{TV} \Bigg{)}^{\frac{1}{2}} \Bigg{]}, \label{eq:assouad weakening}
    \end{align}
    where $P_{v,\pm j}$ is defined as the distribution $P_{\theta_v}$ where coordinate $j$ takes the value $v_j = \pm 1$. The inequality above follows from the Cauchy-Schwarz inequality and convexity of the total variation distance (see \citep[p.165]{Duchi16}).

Define:
\begin{equation}
p_v(z) \coloneqq \frac{1}{1 + \exp(\frac{\delta}{\pi(\theta_v)} z^T v)},
\end{equation}
and let $\text{KL}(p \| q)$ denote the binary KL-divergence between $\text{Bernoulli}(p)$ and $\text{Bernoulli}(q)$. By Pinsker's inequality, we have for any $v, v'$:
\begin{align}
    \|P_{\theta_v}^n - P_{\theta_{v'}}^n\|^2_\text{TV} \leq \frac{1}{4} \Big{[} \text{KL}(P_{\theta_v}^n \| P_{{\theta_{v'}}}^n) + \text{KL}(P_{\theta_{v'}}^n \| P_{\theta_v}^n) \Big{]} = \frac{1}{4} \sum_{i=1}^n \Big{[} \text{KL}(p_v(z_i) \| p_{v'}(z_i)) +  \text{KL}(p_{v'}(z_i) \| p_{v}(z_i)) \Big{]}. 
\end{align}
Letting $p_a \coloneqq 1/(1+e^a)$ and $p_b \coloneqq 1/(1+e^b)$, we have (see \citep[p.167]{Duchi16}):
\begin{equation}
    \text{KL}(p_a \| p_b) + \text{KL}(p_b \| p_a) \leq (a-b)^2.
\end{equation}
This implies:
\begin{align}
    \|P_{\theta_v}^n - P_{\theta_{v'}}^n\|^2_\text{TV} &\leq \frac{1}{4} \sum_{i=1}^n \Bigg{[} \frac{\delta}{\pi(\theta_v)} z_i^T v - \frac{\delta}{\pi(\theta_{v'})} z_i^T v' \Bigg{]}^2 \\
    &= \frac{\delta^2}{4} \sum_{i=1}^n \Bigg{[} z_i^T \Bigg{(} \frac{v}{\pi(\theta_v)} - \frac{v'}{\pi(\theta_{v'})} \Bigg{)} \Bigg{]}^2.
\end{align}

In order to lower bound \eqref{eq:assouad weakening}, we use the preceding bound to note:
\begin{align}
     \frac{1}{2^d d} \sum_{j=1}^d \sum_{v \in \V} \| P_{v,+j} - P_{v,-j} \|^2_\text{TV} &\leq \frac{\delta^2}{4 d 2^d}  \sum_{v \in \V} \sum_{j=1}^d \sum_{i=1}^n \Bigg{[} z_{ij} \Bigg{(} \frac{1}{\pi(\theta_{v_j})} + \frac{1}{\pi(\theta_{v_{-j}})} \Bigg{)} \Bigg{]}^2 \\
     &= \frac{4\delta^2}{ d 2^d}  \sum_{v \in \V} \sum_{j=1}^d \sum_{i=1}^n z_{ij}^2 \lambda_j^2 \\
     &= \frac{4\delta^2}{d} \sum_{j=1}^d \sum_{i=1}^n  z_{ij}^2 \lambda_j^2.
\end{align}
Thus, using \eqref{eq:assouad weakening}, we have:
\begin{align}
    \R_\text{prior}(\pi, L; \Theta) &\geq \frac{d\delta}{2} \Bigg{[} 1 - \Bigg{(} \frac{4\delta^2}{d} \sum_{j=1}^d \sum_{i=1}^n  z_{ij}^2 \lambda_j^2 \Bigg{)}^\frac{1}{2} \Bigg{]}.
\end{align}
Setting 
\begin{equation}
    \delta^2 = \frac{d}{16 \sum_{j=1}^d \sum_{i=1}^n  z_{ij}^2 \lambda_j^2},
\end{equation}
we obtain the desired result:
\begin{align}
    \R_\text{prior}(\pi, L; \Theta) &\geq \frac{d \delta }{4} = \frac{1}{16} \cdot \frac{d^\frac{3}{2}}{\sqrt{\sum_{j=1}^d \sum_{i=1}^n  z_{ij}^2 \lambda_j^2}}.
\end{align}

\end{proof}

\section{Bernoulli mean estimation: upper bounds}
\label{app:upper bounds}

Fig.~\ref{fig:bernoulli upper bounds} presents \emph{upper bounds} on prioritized risk computed using Bayesian inference (with prior $\pi$) for the problem of Bernoulli mean estimation. Specifically, we present the learner-specific prioritized risk where the learner outputs the mean of the posterior distribution computed using the prior $\pi$ and a dataset of size $n$; since $\pi$ is chosen to be a Beta distribution (Beta($\alpha=1$, $\beta=2$)) and the underlying random variable has a Bernoulli distribution, one can analytically perform Bayesian inference in this setting. We estimate the expectation over datasets by averaging the loss over 10,000 datasets. We compare the results with two other learners corresponding to performing (i) performing Bayesian inference with a uniform prior, and (ii) Bayesian inference with a \emph{different prior} (specifically, we use Beta($\alpha=1$, $\beta=4$), which concentrates the prior towards values of $\theta$ where $\pi$ is higher). As the figure indicates, Bayesian inference with $\pi$ achieves a lower learner-specific prioritized risk than Bayesian inference with a uniform prior. However, the figure also shows that the custom inference algorithm (Bayesian inference with a more concentrated prior) achieves a lower learner-specific prioritized risk compared to Bayesian inference with $\pi$. Thus, Bayesian inference with prior $\pi$ is not necessarily optimal from the perspective of the prioritized risk. This observation thus leaves open the interesting direction for future work of identifying algorithms that achieve optimal prioritized risk.   

\begin{figure}[h]
 \begin{center}
 \includegraphics[width=0.6\columnwidth]{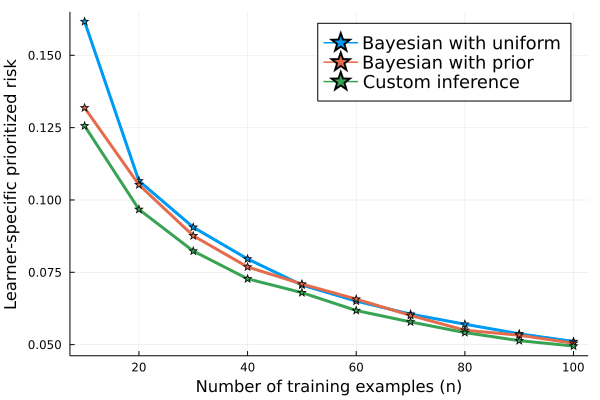}
 \end{center} 
 \vspace{-10pt}
 \caption{Upper bounds on prioritized risk using Bayesian inference with uniform prior, Bayesian inference with prior $\pi$, and a custom inference algorithm (Bayesian inference with a modified prior).  \label{fig:bernoulli upper bounds}}
 \vspace{-15pt}
 \end{figure}


\end{document}